\documentclass{article}

\usepackage{arxiv}

\usepackage{amsmath,amssymb,amsthm}

\newtheorem{theorem}{Theorem}

\usepackage[utf8]{inputenc} 
\usepackage[T1]{fontenc}    
\usepackage{hyperref}       
\usepackage{url}            
\usepackage{booktabs}       
\usepackage{amsfonts}       
\usepackage{nicefrac}       
\usepackage{microtype}      
\usepackage{cleveref}       
\usepackage{lipsum}         
\usepackage{graphicx}
\usepackage{natbib}
\usepackage{doi}

\usepackage{tikz}
\usetikzlibrary{arrows.meta, positioning, shapes.geometric, shapes.symbols, fit, backgrounds, calc}
\usepackage{adjustbox}
\usepackage{multirow}
\usepackage{subcaption}

\usepackage{makecell}

\usepackage[ruled,vlined,linesnumbered]{algorithm2e}
\newcommand{\Select}{\textsc{Select}}
\newcommand{\Decide}{\textsc{Decide}}
\newcommand{\Deduce}{\textsc{Deduce}}
\newcommand{\GetObservation}{\textsc{GetObservation}}
\newcommand{\GetBranchingScore}{\textsc{GetBranchScore}}
\newcommand{\cex}{cex}
\newcommand{\sat}{sat}
\newcommand{\unsat}{unsat}

\newcommand{\problems}{queue}
\SetKwInOut{Input}{input}
\SetKwInOut{Output}{output}
\SetKw{Break}{break}
\SetKw{Continue}{continue}
\SetKw{In}{in}

\makeatletter

\renewcommand{\sectionautorefname}{\S\@gobble}
\renewcommand{\subsectionautorefname}{\S\@gobble}
\renewcommand{\subsubsectionautorefname}{\S\@gobble}
\def\appendixautorefname{\S\@gobble}%

\makeatother

\newcommand{\babverifier}{\textsc{BaB$_{\text{NNV}}$}}
\newcommand{\babsr}{\textsc{BaBSR}}
\newcommand{\neuralsat}{\textsc{NeuralSAT}}
\newcommand{\marabou}{\textsc{Marabou}}
\newcommand{\crown}{\textsc{$\alpha\beta$-Crown}}
\newcommand{\tool}{\textsc{Rsb}}
\newcommand{\fsb}{\textsc{Fsb}}
\newcommand{\polarity}{\textsc{Polarity}}
\newcommand{\random}{\textsc{Random}}
\newcommand{\gcpcrown}{\textsc{Gcp-Crown}}
\newcommand{\upb}{\textsc{Upb}}

\newtoggle{usecomment}
\settoggle{usecomment}{false}

\newcommand{\eg}{\emph{e.g.}}

\title{{\Large Verifying Neural Networks with Reinforcement Learning}}

\author{Hai Duong \\
	Department of Computer Science \\
	George Mason University \\
        Fairfax, VA, USA \\
	\And
	Thanh Le \\
	Unaffiliated \\
	Yokosuka, Japan \\
	\And
	ThanhVu Nguyen \\
	Department of Computer Science \\
	George Mason University \\
        Fairfax, VA, USA \\
}

\date{}

\renewcommand{\undertitle}{}
\renewcommand{\shorttitle}{Verifying Neural Networks with Reinforcement Learning}

\hypersetup{
  pdftitle={Verifying Neural Networks with Reinforcement Learning},
  pdfsubject={q-bio.NC, q-bio.QM},
  pdfauthor={Hai Duong, Thanh Le, ThanhVu Nguyen},
  pdfkeywords={First keyword, Second keyword, More},
}

\begin{document}

\maketitle

\begin{abstract}
  Formal verification can play a key role in ensuring the reliability of Deep Neural Networks (DNNs) deployed in safety-critical systems.
  Modern DNN verifiers employ a branch-and-bound framework,
  which alternates between branching (splitting into smaller subproblems) and
  bounding (pruning subproblems) to efficiently explore the verification space.
  However, existing branching heuristics make greedy decisions based on static scoring functions.
  They do not anticipate long-term efficiency or leverage the growing availability of verification data to improve performance.

  This work introduces \tool{}, a reinforcement learning framework that learns to refine baseline branching heuristics.
  It trains an actor-critic architecture to maximize cumulative future rewards rather than immediate scores.
  The actor generates attention weights from observations of raw neuron features and learned graph embeddings,
  which rescale baseline heuristic scores to guide neuron branching.
  Evaluation on 600 challenging instances demonstrates that \tool{} consistently outperforms state-of-the-art branching heuristics,
  solving 11\% more instances while reducing branch exploration by 50\%.
\end{abstract}

\section{Introduction}
  \label{sec:intro}
  Deep Neural Networks (DNNs) are increasingly being employed as components of mission-critical systems across a range of domains, \eg, autonomous driving~\cite{shao2023safety} or medicine~\cite{morris2023deep}.
  Thus, DNNs require high levels of assurance to be deployed with confidence in such applications.
  To provide evidence that DNN behavior meets expectations,
  researchers have developed techniques for verifying that DNNs
  satisfy required specifications~\cite{shriver2021reducing,duong2026verifying,duong2026verifying2}.
  Dozens of DNN verifiers have been reported in the
  literature, and a yearly competition has documented advances in their capabilities~\cite{brix2023first,brix2024fifth,kaulen20256th}.


  As with traditional software verification, DNN verification is often framed as a satisfiability problem and also suffers from scalability challenges~\cite{katz2017reluplex,duong2026compositional}.
  To tackle the large search space of DNN verification, modern verifiers~\cite{ferrari2022complete,wang2021beta,duong2024harnessing,wu2024marabou} often employ the Branch-and-Bound (BaB)~\cite{bunel2020branch} approach, which alternates between \emph{bounding} (employing abstraction methods to prune branches) and \emph{branching} (selecting neurons to split). Most advances in DNN verification focus on developing novel abstract domains to tighten the bounds computed during the bounding step~\cite{zhang2018efficient,wang2018formal,singh2019abstract,singh2018fast,tran2019star,xu2020automatic,xu2020fast}.

  In contrast, the branching step---which determines the order of neuron splits and therefore significantly impacts BaB verifier efficiency---is under-explored.
  State-of-the-Art (SOTA) DNN verifiers often adopt branching heuristics like \babsr{}~\cite{bunel2020branch}
  and \fsb{}~\cite{de2021improved} with hand-crafted scoring functions to determine the branching order.
  These heuristics have several major limitations:
  (i) they require specialized expertise to design for different problem characteristics (\eg, different activation functions);
  (ii) they make greedy decisions at each step without anticipating long-term search efficiency; and
  (iii) they fail to leverage the growing availability of verification data to improve the performance.

  This work introduces \tool{}, a Deep Reinforcement Learning (DRL) approach that addresses these limitations by learning adaptive branching policies from verification data.
  First, \tool{} automatically learns to refine any baseline branching heuristic such as \fsb{} by training on diverse verification instances.
  Second, \tool{} formulates branching as a sequential decision-making problem~\cite{sutton1998reinforcement} and trains a network~\cite{bello2017neural} via DRL~\cite{mnih2016asyncchronous,haarnoja2018soft} to maximize cumulative future rewards (\eg, negation of numbers of branches), which enables long-term search efficiency optimization instead of greedy local decisions.
  Third, by leveraging available verification benchmarks~\cite{brix2024fifth,kaulen20256th} and the ease of synthesizing additional training data, \tool{} follows the paradigm of modern ML where performance improves with scale as demonstrated by large language models~\cite{kaplan2020scaling}.
  Finally,  \tool{} serves as a general architecture-agnostic framework that can enhance any existing branching heuristic across diverse DNN types and activation functions by combining learned graph embeddings (to encode network structure) and raw neuron features (\eg, neuron bounds, to capture properties/specifications).

  Evaluation on 600 challenging instances across networks of varying sizes and architectures shows that
  \tool{} consistently outperforms SOTA branching heuristics such as \fsb{} by solving 11\% more instances and reducing branch exploration by 50\%.
  Furthermore, \tool{} generalizes beyond its training distribution, solving the most instances on unseen network architectures (\eg, CNN) and activation function (\eg, Sigmoid).
  We consider these improvements substantial given the maturity of existing heuristics like \fsb{}, which have been widely adopted by leading verifiers and represent the culmination of years of research.


  Our contributions are:
  (i) a DRL formulation that learns to refine  existing baseline branching heuristic through attention weights (thereby enabling architecture-agnostic policy learning); 
  (ii) an architecture that handles variable-sized inputs as the number of unstable neurons varies across instances and shrinks with each BaB split;
  (iii) a training procedure adapted to BaB's branching dynamics that accounts for both successor subproblems per branching decision;
  and (iv) an implementation of \tool{} using an existing DNN verifier and an evaluation showing the effectiveness of the approach.

\section{Preliminaries}
  \label{sec:preliminary}
  \noindent\textbf{DNN verification.}
    Given a DNN $N$ and a property $\phi$, the \emph{DNN verification problem} asks if $\phi$ is a valid property of $N$.
    Typically, $\phi$ is a formula of the form $\phi_{in} \Rightarrow \phi_{out}$, where $\phi_{in}$ is a property over the inputs of $N$ and $\phi_{out}$ is a property over the outputs of $N$.
    Modern techniques often treat the DNN verification as a \emph{satisfiability} problem~\cite{wang2021beta,ferrari2022complete,duong2023dpllt,wu2024marabou,duong2024harnessing}.
    Specifically, given a formula $\alpha$ representing $N$ and the formula $\phi_{in}\Rightarrow \phi_{out}$ representing the desired property, a DNN verifier checks the satisfiability of the formula:
    \(
    \alpha \land \phi_{in} \land \overline{\phi_{out}}
    \).
    The verifier returns \texttt{unsat} if the formula is unsatisfiable, indicating $\phi$ is a valid property of $N$, and \texttt{sat} otherwise, indicating $\phi$ is not valid.
    An \texttt{unsat} result can further be accompanied by a proof, \eg, derived from the BaB search tree, that can be independently checked~\cite{duong2026generating}.

  Modern DNN verifiers often follow a branch-and-bound (BaB) framework~\cite{bunel2020branch}, shown in
    \autoref{alg:babnnv} in~\autoref{apdx:bab_algorithm}). BaB iterates between two components: \Decide{} (branching) assigns active/inactive status for a neuron and thereby splits the problem into two subproblems, and \Deduce{} (bounding) computes the bounds of neuron outputs and checks the feasibility of the current subproblem. Both components are critical to verification efficiency, but this work focuses on improving \Decide{}, which is responsible for branching decisions and his relatively under-explored compared to \Deduce{} (e.g., designing new abstract domains~\cite{singh2018fast,singh2019abstract,tran2019star,wang2018formal,xu2020automatic,xu2020fast,zhang2022general,duong2024harnessing,ferrari2022complete}).

  \noindent\textbf{Branching (\Decide{}) heuristics.}
    At each verification step, selecting which neuron to branch on from the set of all neurons constitutes a combinatorial optimization problem.
    The choice of which neuron to split affects how quickly the search space is explored and pruned, \eg, determining the numbers of steps and the verification time.
    A well-designed decision heuristic can dramatically reduce the search space, while a poor one may lead to exponential growth in numbers of subproblems.

    SOTA heuristic \fsb{}~\cite{de2021improved} estimates improvement for each neuron if it is branched on, and is widely used by leading DNN verifiers. 
    However, \fsb{} only considers local optimality at each step, not long-term search efficiency, and relies on domain expert knowledge, \eg, Lagrangian computations.
    In contrast, \tool{} learns to generate weights that modify baseline heuristic's scores (\eg, \fsb{}) to optimize long-term verification efficiency, thus,
    produces more effective decisions overall (see~\autoref{sec:results}).
  \noindent\textbf{Deep Reinforcement Learning.}
    DRL combines Reinforcement Learning (RL) with DNN to learn sequential decision-making policies from interaction with an environment.
    An RL problem is formalized as a Markov Decision Process $\langle S, A, P, R \rangle$, where $S$ is the state space, $A$ the action space, $P$ the transition function, and $R(s, a)$ the reward function.
    A DRL agent learns a policy to maximize the expected cumulative reward over time.
    This work uses an actor-critic algorithm~\cite{haarnoja2018soft}, which maintains two networks: an actor that selects actions and a critic that estimates their value.
    The actor is updated to take actions with higher estimated value, while the critic is updated to better predict future rewards.
    We refer to \autoref{apdx:drl} for a detailed actor and critic objectives.

\section{Motivating example}\label{sec:motiv}

  \begin{figure}[t]
    \centering
    \begin{subfigure}[t]{0.44\linewidth}
      \centering
      \includegraphics[width=\linewidth]{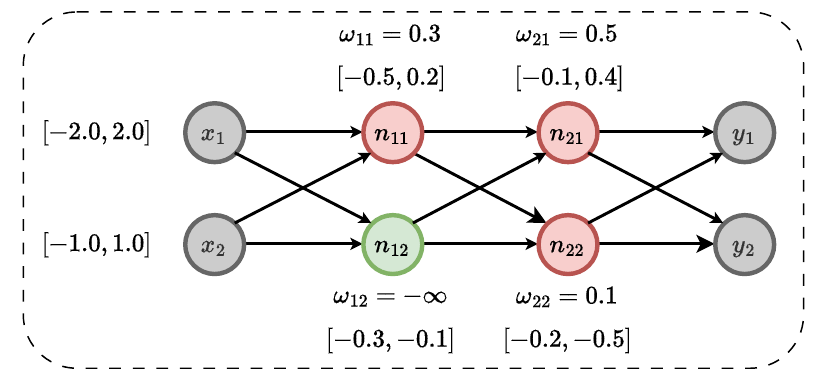}
      \caption{A DNN with neuron bounds computed by an abstraction and heuristic scores $\omega_{ij}$ computed by \fsb{}.
      }
      \label{fig:example:dnn}
    \end{subfigure}%
    \hfill
    \begin{subfigure}[t]{0.25\linewidth}
      \centering
      \includegraphics[width=\linewidth]{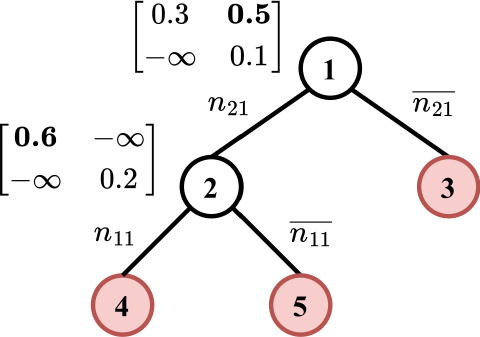}
      \caption{A search tree with \fsb{}}
      \label{fig:example:score}
    \end{subfigure}%
    \begin{subfigure}[t]{0.30\linewidth}
      \centering
      \includegraphics[width=\linewidth]{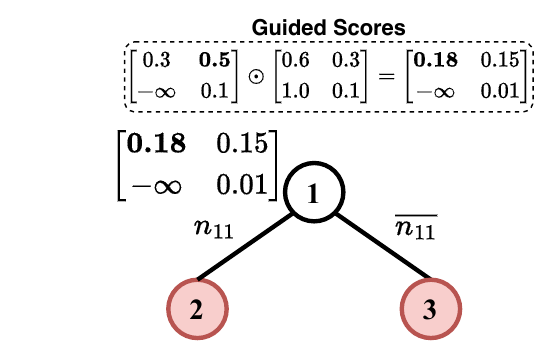}
      \caption{A search tree with \tool{}}
      \label{fig:example:score-guided}
    \end{subfigure}%
    \caption{Comparison of different search trees performed by \fsb{} and \tool{} for verifying the property in~\autoref{eq:example:property}.}
    \label{fig:example}
  \end{figure}

  We illustrate \tool{} using a concrete example.
  Consider the DNN $N$ shown in~\autoref{fig:example:dnn}, which consists of two linear layers with ReLU activations.
  We aim to verify the property:
  \begin{equation}\label{eq:example:property}
    \phi \equiv
    (-2 \le x_1 \le 2 \land -1 \le x_2 \le 1) \Rightarrow (y_1 > y_2)
  \end{equation}
  \babverifier{} computes a bound estimation of hidden neurons (\eg, using an abstraction as \textsc{LiRPA}~\cite{xu2020automatic}) to determine neuron stability.
  This step reveals three \emph{unstable} ReLU neurons (shown in red in~\autoref{fig:example:dnn}): $n_{11}$, $n_{21}$, and $n_{22}$.
  For example, $n_{11}$ is unstable because its bounds $[-0.5, 0.2]$ mean it could be either active or inactive depending on the input.
  \babverifier{} splits (branches) each unstable neuron into positive (active) and negative (inactive) branches.


  SOTA \babverifier{} employs various branching heuristics to guide the search, \eg, \fsb{}~\cite{de2021improved}.
  A branching heuristic assigns scores $\omega$ to unstable neurons
  to prioritize which neuron to split next.
  \autoref{fig:example:score} shows a search tree of verifying the property in~\autoref{eq:example:property} using \babverifier{} with \fsb{}, a popular heuristic used by major DNN verifiers including \crown{}~\cite{wang2021beta}, \gcpcrown{}~\cite{zhang2022general} and \neuralsat{}~\cite{duong2024harnessing}.

  First, \babverifier{} computes \fsb{} scores for unstable neurons as $\omega_{11} = 0.3$, $\omega_{21} = 0.5$, and $\omega_{22} = 0.1$
  and selects $n_{21}$ for the initial branching because it has the highest score.
  After this branching, it proves unsatisfiability of the negative branch $\overline{n_{21}}$ while the positive branch is still undecided.
  Next, \babverifier{} computes \fsb{} scores for the remaining two unstable neurons and obtains $\omega_{11} = 0.6$ and $\omega_{22} = 0.2$.
  and selects $n_{11}$ for the second branching decision based on its score.
  \babverifier{} then can prove both branches and completes the verification.
  In total, \babverifier{} with \fsb{} requires 2 steps and processes 4 branches to verify the property.

  \tool{} improves this by learning to generate new weights to adjust the original \fsb{} scores.
  Neuron-level features are constructed from four available \emph{raw features}
  of unstable neurons $n_{11}$, $n_{21}$, $n_{22}$:
  lower bound, upper bound, bias ($0.0$ for simplicity), and mask (1 for unstable, 0 for stable):
  \begin{equation}\label{eq:observation_construction}
    \mathbf{x} = \GetObservation\left(
      \begin{bmatrix}
        -0.5 \\
        0.2 \\
        0.0 \\
        1.0 \\
      \end{bmatrix},
      \begin{bmatrix}
        -0.1 \\
        0.4 \\
        0.0 \\
        1.0 \\
      \end{bmatrix},
      \begin{bmatrix}
        -0.2 \\
        0.5 \\
        0.0 \\
        1.0 \\
      \end{bmatrix}
    \right)
  \end{equation}
  $\GetObservation{}$ combines neuron features with Graph Neural Networks (GNN) embeddings to produce observation $\mathbf{x}$ (see~\autoref{sec:training:embed}).
  Next, the actor $\pi_\psi$ takes as input $\mathbf{x}$ and generates weights:
  \begin{equation}\label{eq:weight_generation}
    \begin{bmatrix} a_{n_{11}} \\ a_{n_{21}} \\ a_{n_{22}} \end{bmatrix} = \pi_\psi\left(\mathbf{x}\right) = \begin{bmatrix}
      0.6 \\
      0.3 \\
      0.1
    \end{bmatrix}
  \end{equation}
  The precomputed \fsb{} scores $\omega_{ij}$ are then scaled by the weights $a_{ij}$ to produce the guided scores:
  \begin{equation}\label{eq:score_multiplication}
    \begin{bmatrix} \omega'_{11} \\ \omega'_{21} \\ \omega'_{22} \end{bmatrix}
    = \begin{bmatrix} a_{n_{11}} \\ a_{n_{21}} \\ a_{n_{22}} \end{bmatrix} \odot \begin{bmatrix} \omega_{11} \\ \omega_{21} \\ \omega_{22} \end{bmatrix}
    = \begin{bmatrix} 0.6 \\ 0.3 \\ 0.1 \end{bmatrix} \odot \begin{bmatrix} 0.3 \\ 0.5 \\ 0.1 \end{bmatrix}
    = \begin{bmatrix} 0.18 \\ 0.15 \\ 0.01 \end{bmatrix}
  \end{equation}
  This reweighting flips the priority, \eg, $n_{11}$ now has the highest score ($\omega'_{11} = 0.18$) instead of $n_{21}$ ($\omega'_{21} = 0.15$).
  As shown in~\autoref{fig:example:score-guided}, branching on $n_{11}$ first immediately proves both branches (positive $n_{11}$ and negative $\overline{n_{11}}$ are UNSAT), completing the verification in just one step (2 branches) instead of two steps (4 branches) when using \fsb{}.

  \noindent\textbf{Key intuition.} By learning from existing DNN verification tasks, \tool{} determines that certain neurons tend to be more effective branching points than others.
  For example, branching on $n_{11}$ is preferred over $n_{21}$ because $n_{11}$ lies in an earlier layer and has a wider pre-activation interval, causing its split to simultaneously stabilize multiple downstream neurons.
  In particular, branching on $n_{11}$ constrains both $n_{21}$ and $n_{22}$ in both subproblems, leading to tighter intermediate bounds that allow verification to terminate quicker.
  In contrast, $n_{21}$ (chosen by \fsb{}) is only weakly coupled to $n_{22}$; branching on it leaves $n_{22}$ unstable, requiring additional splits.
  Through training on many networks and properties, \tool{} learns this recurring pattern and selects neurons that are more strongly connected,
  even if their immediate heuristic scores (\eg, computed by \fsb{}) are not maximal.

\section{Branching as DRL environment}
  \label{sec:propose}

  We map the branching problem in \babverifier{} to a DRL environment~\cite{sutton1998reinforcement}, defining a state space $\mathcal{S}$, observation space $\mathcal{O}$, action space $\mathcal{A}$, and reward function $R$.

  \noindent\textbf{State space.}
  Each state $\mathbf{s} \in \mathcal{S}$ comprises the verification instance $\langle N, \phi_{in}, \phi_{out} \rangle$ and the current assignment $\sigma$ encoding accumulated branching decisions.
  When \babverifier{} branches on neuron $n_{ij}$, it generates two successor states:
  \begin{align}
    \mathbf{s}' &= \langle N, \phi_{in}, \phi_{out}, \sigma \land n_{ij} \rangle \quad \text{(positive branch)} \\
    \bar{\mathbf{s}}' &= \langle N, \phi_{in}, \phi_{out}, \sigma \land \overline{n_{ij}} \rangle \quad \text{(negative branch)}
  \end{align}
  where $n_{ij}$ and $\overline{n_{ij}}$ denote the neuron being active and inactive, respectively.
  This dual-branch structure means each branching decision expands the search tree exponentially, which distinguishes DNN verification from single-successor combinatorial optimization problems.

  \noindent\textbf{Observation space.}
  The full state $\mathbf{s}$ cannot be directly processed by the DRL agent due to its complex, variable-sized structure.
  We therefore define an observation space $\mathcal{O}$ and an extraction function \GetObservation{} that maps $\mathbf{s}$ to a compact, agent-compatible representation satisfying three criteria:
  (1) \emph{generalizability} across DNN architectures,
  (2) \emph{expressivity} to capture neuron interdependencies, and
  (3) \emph{scalability} by attending only to unstable neurons.

  Formally, $\mathcal{O} \equiv \mathbb{R}^{|\mathcal{N}_u| \times (F + E)}$, where $\mathcal{N}_u$ is the set of unstable neurons, and $F$, $E$ are the raw feature and embedding dimensions.
  For each unstable neuron $n_{ij}$, raw features include the activation bounds $(l_{ij}, u_{ij})$, bias $b_{ij}$, and a binary mask $m_{ij}$ indicating whether the neuron has been branched on previously.
  These local features are augmented with learned embeddings from a GNN that encodes the global structure of the verification problem (\autoref{sec:training:embed}).

  \noindent\textbf{Action space.}
  Existing branching heuristics such as \babsr{} and \fsb{} rank unstable neurons by computing a score $\omega_{ij}$ for each and greedily selecting $n^* = \text{argmax}_{ij}\,\omega_{ij}$, yielding locally optimal decisions.
  \tool{} improves on this by learning to adjust these scores rather than replacing them.
  We define the action space as $\mathcal{A} \equiv [0, 1]^{|\mathcal{N}_u|}$, where each $a_{ij} \in [0, 1]$ is a learned weight for neuron $n_{ij}$.
  Given observation $\mathbf{o}$, the actor samples $\mathbf{a} \sim \pi_\psi(\mathbf{a}|\mathbf{o})$ and selects:
  \begin{equation}\label{eq:guided}
    n^* = \underset{i, j}{\text{argmax }} \, \omega_{ij} \odot a_{ij}
  \end{equation}
  This allows the agent to refine any heuristic's priorities based on the current state and promote long-term branching efficiency.

  \noindent\textbf{Reward function.}
  The reward function $R(\mathbf{s}, \mathbf{a})$ incentivizes efficient verification by penalizing unverified subproblems.
  Formally, the agent incurs a unit reward for each resolved subproblem:
  \begin{equation}
    R(\mathbf{s}, \mathbf{a}) = \sum_{\sigma' \in \{\sigma \wedge n^*, \sigma \wedge \overline{n^*}\}} \neg \textsc{Deduce}(N, \phi_{in}, \phi_{out}, \sigma')
  \end{equation}
  where $\langle N, \phi_{\text{in}}, \phi_{\text{out}}, \sigma \rangle \leftarrow \mathbf{s}$, and $n^*$ is selected neuron obtained from $\mathbf{a}$ (\autoref{eq:guided}).
  Since each verified subproblem contributes a $+1$ reward, maximizing the cumulative discounted reward is equivalent to minimizing the total number of subproblems generated throughout the verification process.
  The overall goal is to learn a policy $\pi_\psi(\mathbf{a}|\mathbf{o})$ that maximizes the expected cumulative discounted reward $\mathbb{E}_{\pi_\psi}\bigl[\sum_{t=0}^T \gamma^t R(\mathbf{s}_t, \mathbf{a}_t)\bigr]$, where $\gamma$ is the discount factor.



  \begin{theorem}[Sound and Complete]
    \label{thm:soundness-completeness}
    \tool{} preserves \babverifier{}'s soundness and completeness.
  \end{theorem}
  \begin{proof}
    \tool{} only affects the neuron selection order via \autoref{eq:guided}, leaving the underlying verification logic of \babverifier{} intact.
    Since \babverifier{} exhaustively explores all necessary branches regardless of branching order, the verification result is unchanged.
  \end{proof}

\section{The \tool{} approach}
  \begin{figure}
    \begin{minipage}[t]{0.48\linewidth}
  \begin{algorithm}[H]
    \small
    \Input{DNN $N$, property $\phi$, trained actor $\pi_\psi$}
    \Output{\unsat{} if property is valid, otherwise \sat{}.}

    \BlankLine

    $\problems \gets \{ \emptyset \}$ \\
    \While{$\problems$}{
      $\sigma \gets \Select(\problems)$ \\
      \If{\Deduce($N, \phi, \sigma$)}{
        $\mathbf{o}_t \gets \GetObservation(N, \phi, \sigma)$ \label{line:get-observation} \\
        $\mathbf{a}_t \sim \pi_\psi(\mathbf{a}|\mathbf{o}_t)$ \label{line:drl-decide} \\
        $\omega_t \gets \GetBranchingScore(N, \phi, \sigma)$ \label{line:get-branching-score}\\
        $n^* \gets \arg\max_{ij} \mathbf{a}_t \odot \omega_t$ \label{line:select-neuron} \\
        $\problems \gets \problems \cup \{ \sigma \land n^* ~;~ \sigma \land \overline{n^*} \}$ \\
      }
    }
    \Return{\unsat{}}
    \caption{\babverifier{} with \tool{}}
    \label{alg:inference}
  \end{algorithm}
  \end{minipage}%
  \hfill
  \begin{minipage}[t]{0.48\linewidth}

    \begin{algorithm}[H]
      \small
      \caption{Training \tool{}}
      \label{alg:training}
      \Input{Training set $\{(N_i, \phi_i)\}$} 
      \Output{Trained actor $\pi_{\psi}$, critics $Q_{\phi_1}, Q_{\phi_2}$}

      $\mathcal{D} \gets \emptyset$ \\
      \For{$e ~\In~ [1, ..., \text{NumEpisodes}]$}{
        $(N, \phi) \gets \textsc{SampleInstance}(\{(N_i, \phi_i)\})$ \label{line:train-start} \\
        $\mathbf{o}_0 \gets \GetObservation(N, \phi)$ \\ \label{line:training:get-observation}
        \For{$t\;\In\;[0, ..., T]$}{
          $\omega_t \gets \GetBranchingScore(N, \phi_i, \sigma)$ \\
          $\mathbf{a}_t \sim \pi_{\psi}(a|\mathbf{o}_t)$ \\ \label{line:training:attention-weight}
          $n^* \gets \arg\max_{ij} \mathbf{a}_t \odot \omega_t$ \\ \label{line:training:decision}
          $\mathbf{o}'_{t+1}, \bar{\mathbf{o}}'_{t+1} \gets \Decide(n^*)$ \\
          $m_{t+1}, \bar{m}_{t+1} \gets \Deduce(\mathbf{o}'_{t+1}, \bar{\mathbf{o}}'_{t+1})$ \\
          $r_t \gets -|\{\mathbf{o}' : m' = 1\}|$ \\ \label{line:training:reward}
          $\mathcal{D}.add(\{(\mathbf{o}_t, \mathbf{a}_t, r_t, \mathbf{o}'_{t+1}, \bar{\mathbf{o}}'_{t+1}, m_{t+1}, \bar{m}_{t+1})\})$ \\ \label{line:training:store-transition}
          $\pi_{\psi}, Q_{\phi_1}, Q_{\phi_2} \gets \textsc{Optimize}(\mathcal{D})$ \label{line:update-sac} \\
        }
      }
      \Return{$\pi_{\psi}, Q_{\phi_1}, Q_{\phi_2}$}
    \end{algorithm}

  \end{minipage}
  \end{figure}

  \autoref{alg:inference} shows how \tool{} integrates into \babverifier{}.
  At each branching step, \GetObservation{} extracts the current observation $\mathbf{o}_t$ (\autoref{line:get-observation}), the actor $\pi_\psi$ generates attention weights $\mathbf{a}_t$ (\autoref{line:drl-decide}), and \GetBranchingScore{} computes baseline scores $\omega_t$ for unstable neurons (\autoref{line:get-branching-score}).
  The neuron maximizing the guided score $\arg\max_{ij} \mathbf{a}_t \odot \omega_t$ is selected for branching (\autoref{line:select-neuron}); the rest of the procedure is unchanged from \babverifier{}.

  \tool{} primarily relies on three components: a GNN that enriches raw neuron features with structural embeddings (\autoref{sec:training:embed}), an actor-critic architecture that learns to refine branching scores (\autoref{sec:training:architecture}), and an off-policy training procedure with modifications for the dual-branch structure of \babverifier{} (\autoref{sec:training:procedure}).

  \subsection{Observation construction}
  \label{sec:training:embed}

  We train a Graph Convolution Networks (GCN)~\cite{kipf2017semisupervised} to enrich raw neuron features with learned embeddings that capture network-wide structural patterns.
  The graph is constructed with all neurons as nodes, activation bounds, bias, and branching status ($l_{ij}, u_{ij}, b_{ij}, m_{ij}$) as node features, network connections as edges forming adjacency matrix $A$, and heuristic scores (\eg, \fsb{}) as supervision labels for unstable neurons.
  This yields a node regression problem where the GCN predicts branching scores for unstable neurons $\mathcal{N}_u$, producing embeddings $\mathbf{o}_{\text{embed}} \in \mathbb{R}^{|\mathcal{N}_u| \times E}$.

  Since the raw feature dimension $F \ll E$, directly concatenating raw features with embeddings would bias representations toward global information.
  We therefore project raw features $\mathbf{o}_{\text{raw}} \in \mathbb{R}^{|\mathcal{N}_u| \times F}$ to dimension $E$, concatenate with $\mathbf{o}_{\text{embed}}$, and project again to produce the unified observation $\mathbf{x} \in \mathbb{R}^{|\mathcal{N}_u| \times E}$ that balances local and global information.

  \subsection{Actor-Critic architecture}
  \label{sec:training:architecture}

    We employ an actor-critic~\cite{haarnoja2018soft} framework for its sample efficiency~\cite{wen2021characterizing,tan2025actorcritics}, which is important since each verification episode is computationally expensive~\cite{katz2017reluplex}.
    Both actor $\pi_\psi(\mathbf{a}|\mathbf{o})$ and critics $Q_\phi(\mathbf{o},\mathbf{a})$ use PointerNet~\cite{bello2017neural} architectures that naturally handle variable-length inputs (the number of unstable neurons changes across instances), taking the unified observation $\mathbf{x}$ from \autoref{sec:training:embed} as input.

    The actor (\autoref{fig:training:actor}) uses an encoder-decoder with an attention-based pointer mechanism.
    An encoder processes $\mathbf{x}$ to produce reference vectors $\mathbf{r} \in \mathbb{R}^{|\mathcal{N}_u| \times H}$ capturing relationships among all neurons.
    The decoder, initialized with a learnable state $\mathbf{x}_0^{\text{dec}} \in \mathbb{R}^E$, generates query $\mathbf{q} \in \mathbb{R}^{H}$, which PointerNet uses to compute attention scores output as $\mathbf{a} \in \mathbb{R}^{|\mathcal{N}_u|}$.
    The feature of the selected neuron becomes the next decoder state, enabling decisions that condition on prior selections.

    The critic (\autoref{fig:training:critic}) mirrors the actor's encoder but differs in decoder initialization.
    Instead of a learnable state, the decoder is initialized with the actor's decoder state $\mathbf{x}_t^{\text{dec}}$ from the current branching step.
    This produces Q-values $Q(\mathbf{o}, \mathbf{a}) \in \mathbb{R}^{|\mathcal{N}_u|}$ for all neurons simultaneously.
    Twin critics $Q_{\phi_1}$, $Q_{\phi_2}$ share this architecture with independent parameters to reduce variance in value estimation.

    \begin{figure}[t]
      \begin{subfigure}{0.50\linewidth}
        \centering
        \includegraphics[width=\linewidth]{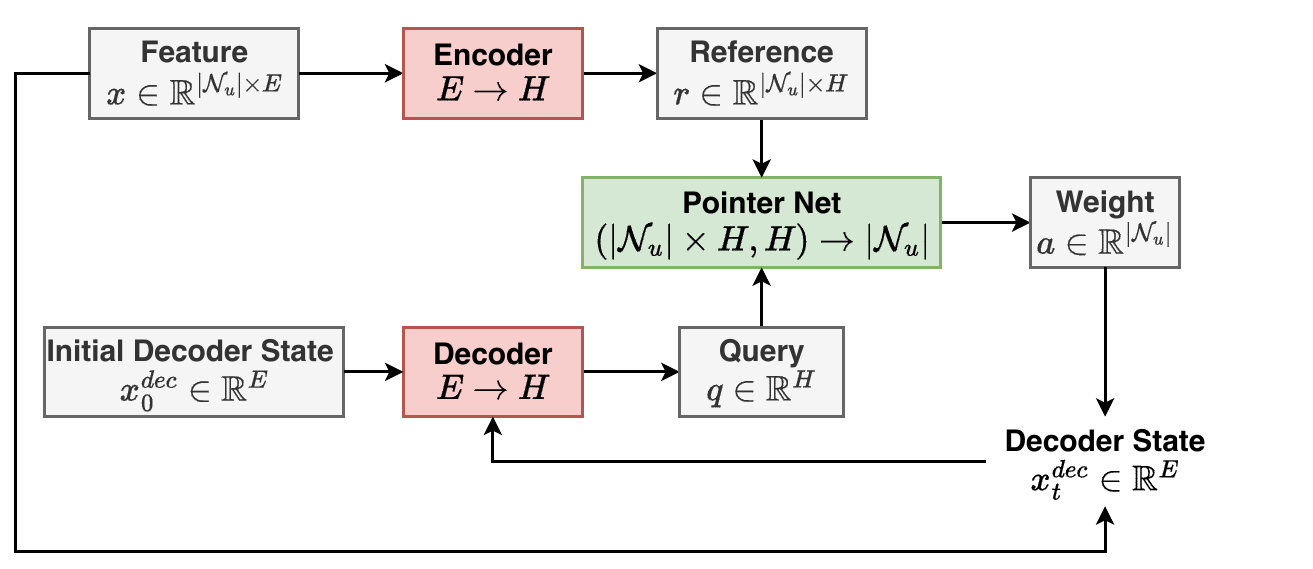}
        \caption{Actor architecture}
        \label{fig:training:actor}
      \end{subfigure}%
      \hfill
      \begin{subfigure}{0.50\linewidth}
        \centering
        \includegraphics[width=\linewidth]{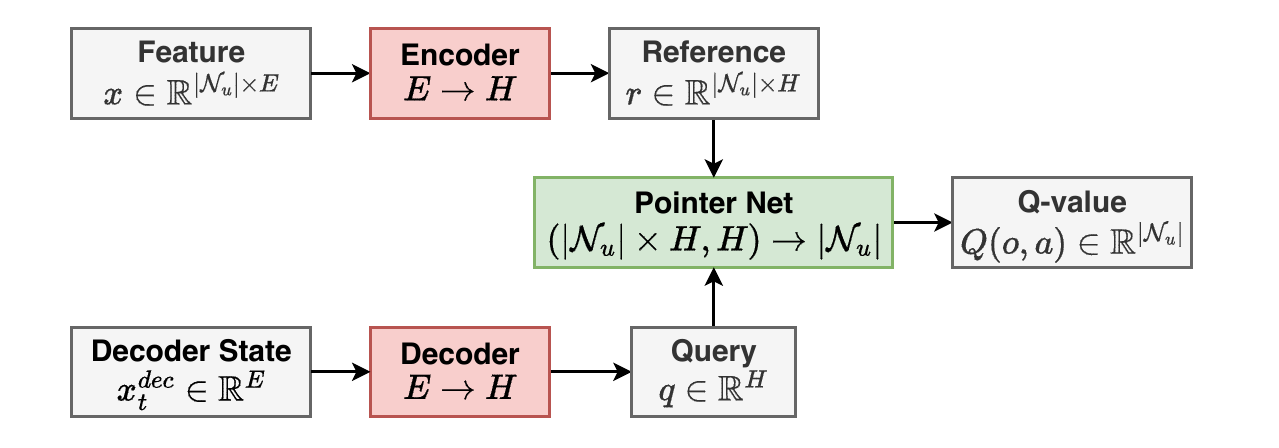}
        \caption{Critic architecture}
        \label{fig:training:critic}
      \end{subfigure}%
      \caption{The actor and critic architectures in \tool{}.}
    \end{figure}

  \subsection{Training procedure}
  \label{sec:training:procedure}

    \autoref{alg:training} presents the training procedure.
    Some functions of \babverifier{} remain (\eg, $\Select$), thus, are hidden for simplicity.
    For each episode, a verification instance $(N, \phi)$ is sampled from the training set (\autoref{line:train-start}) and the initial observation $\mathbf{o}_0$ is extracted (\autoref{line:training:get-observation}).
    At each branching step, the actor samples attention weights $\mathbf{a}_t \sim \pi_\psi(\mathbf{a}|\mathbf{o}_t)$ (\autoref{line:training:attention-weight}), which are combined with heuristic scores $\omega_t$ to select the branching neuron $n^*$ (\autoref{line:training:decision}).
    \textsc{Decide} branches on $n^*$ and returns two successor observations $\mathbf{o}'_{t+1}$, $\bar{\mathbf{o}}'_{t+1}$.
    Next, \textsc{Deduce} runs abstraction on each branch and returns termination masks $m_{t+1}$, $\bar{m}_{t+1}$ indicating which branches are solved.
    The reward $r_t$ penalizes unsolved branches (\autoref{line:training:reward}), and a transition is stored in the replay buffer $\mathcal{D}$ (\autoref{line:training:store-transition}).
    \textsc{Optimize} then samples a batch from $\mathcal{D}$ and updates the actor and critic networks (\autoref{line:update-sac}).
    Two adaptations are needed to handle the dual-branch structure of \babverifier{}.

    \noindent\textbf{Replay buffer.}
      We extend the standard transition~\cite{haarnoja2018soft} to capture both successor subproblems:
      \begin{equation}\label{eq:replay-buffer}
        \mathcal{D} = \{\ldots,\langle \mathbf{o}_t, \mathbf{a}_t, r_t, \mathbf{o}'_{t+1}, \bar{\mathbf{o}}'_{t+1}, m_{t+1}, \bar{m}_{t+1}\rangle, \ldots\}
      \end{equation}
      where $\mathbf{o}'_{t+1}$, $\bar{\mathbf{o}}'_{t+1}$ are the positive and negative branch observations, and $m_{t+1}, \bar{m}_{t+1} \in \{0,1\}$ are termination masks indicating whether each branch is immediately resolved by \textsc{Deduce}.

    \noindent\textbf{Loss function.}
    Since each branching action produces two successor subproblems, the critic loss adapts the Bellman equation to incorporate both branches weighted by their termination masks:
    \begin{align}
      \mathcal{L}_{Q}(\phi) = \mathbb{E}_{\tau_t \sim \mathcal{D}} \left[
        \frac{1}{2} \left(
          Q_\phi(\mathbf{o}_t, \mathbf{a}_t) \right. \right.
      & \left. \left. - \left(
            r_t + \gamma \left(
            m_{t+1} \mathbb{E}_{\mathbf{a} \sim \pi_\psi(\mathbf{a}|\mathbf{o}_{t+1})} \big[Q_{\tilde{\phi}}(\mathbf{o}_{t+1}, \mathbf{a})\big] \right. \right. \right. \right. \nonumber \\
      & \left. \left. \left. \left. + \bar{m}_{t+1} \mathbb{E}_{\mathbf{a} \sim \pi_\psi(\mathbf{a}|\bar{\mathbf{o}}_{t+1})} \big[Q_{\tilde{\phi}}(\bar{\mathbf{o}}_{t+1}, \mathbf{a})\big] \right)
          \right)
        \right)^2
      \right]
      \label{eq:critic-loss-nnv}
    \end{align}
    where $m_{t+1} = 0$ means the branch is solved by bounding step and contributes no future value.
    The target network $Q_{\tilde{\phi}}$ is updated via exponential moving average to stabilize temporal difference target.
    The actor is updated using policy gradients to maximize expected $Q$-values:
    \begin{equation}
      \mathcal{L}_{\pi}(\psi) = \mathbb{E}_{\mathbf{o}_t \sim \mathcal{D}} \left[
        \mathbb{E}_{\mathbf{a} \sim \pi_\psi(\mathbf{a}|\mathbf{o}_t)} \left[
          - Q_\phi(\mathbf{o}_t, \mathbf{a})
        \right]
      \right]
      \label{eq:actor-loss}
    \end{equation}
    Training alternates between updating the critic to improve value estimates and updating the actor to select actions with higher estimated returns.

\section{Evaluation}
\label{sec:results}
  \subsection{Experimental design}
  \label{sec:exp_design}
    \noindent\textbf{Dataset and training.}
      To train \tool{}, we generated 480,960 FNN-based verification instances with varying layers and perturbation radii.
      To select challenging instances that require non-trivial branching\footnote{In fact, out of 2700 instances in regular track of VNN-COMP'25~\cite{kaulen20256th}, 92\% problems (2473/2700) were solved without branching or very few branching steps.}, we applied a two-stage filter using: (i) adversarial attack \textsc{Pgd}~\cite{madry2017towards} to remove SAT instances, and (ii) abstraction \textsc{LiRPA}~\cite{xu2020automatic} to remove easy UNSAT instances.
      This yielded 14,454 challenging instances used for training.
      We train \tool{} for 100,000 episodes with batch size 512 and gradient accumulation over 8 steps (\autoref{tab:train:hyperparameters}) and a 2-layer GCN with 128 hidden dimensions for embeddings.

    \noindent\textbf{Verification problems.}
      We use 600 challenging verification problems across three network architectures (\autoref{tab:benchmarks}).
      In addition to 200 FNN-based instances, we evaluate on 200 CNN-based~\cite{de2021improved} and 200 Sigmoid-based~\cite{kaulen20256th} instances, whose networks are not used in training, for generalization.

      \begin{table}[t]
        \centering
        \small
        \caption{Verification instances.}
        \label{tab:benchmarks}
        \setlength{\tabcolsep}{6pt}
        \vspace{6pt}
        \begin{tabular}{cccccc}
          \toprule
          \textbf{Benchmark} & \textbf{Layers} & \textbf{Types} & \textbf{Neurons} & \textbf{Parameters} & \textbf{Instances} \\
          \midrule
          FNN & 4-6 & FC, ReLU & 0.5K-1.5K & 269K-532K & 200 \\
          CNN & 4-6 & Conv, FC, ReLU & 3.2K-6.8K & 55K-215K & 200 \\
          Sigmoid & 5 & FC, ReLU, Sigmoid & 0.3K & 343K & 200 \\
          \midrule
          \textbf{Total} & & & & & \textbf{600} \\
          \bottomrule
        \end{tabular}
      \end{table}




    \noindent\textbf{Baselines.}
      We compare \tool{} against (i) \fsb{}~\cite{de2021improved}, which estimates bound improvements via Lagrangian Decomposition and is used in \crown{}~\cite{wang2021beta,zhang2022general} (ranked 1st in VNN-COMPs---the annual DNN verification competition to compare different approaches and showcase the latest advances in the field~\cite{bak2021second,brix2023first,kaulen20256th,brix2024fifth});
      and \neuralsat{}~\cite{duong2024harnessing,duong2025neuralsat,duong2025neuralsat2} (2nd in VNN-COMP'24 and '25); (ii) \upb{}~\cite{de2022ibp}, a fast approximation of \fsb{}; and (iii) \polarity{}~\cite{wu2020parallelization}, which uses bound differences directly and is used in \marabou{}~\cite{wu2024marabou} (ranked 2nd in VNN-COMP'23).
      All methods are integrated into \neuralsat{} for a fair comparison.


    \noindent\textbf{Experimental setup.}
      Our experiments were conducted on a Linux machine with an AMD Ryzen Threadripper 32-Core, 128GB RAM, and an NVIDIA GeForce RTX 4090, 24GB VRAM.
      We use timeout 120s per instance, which follows the rules in recent VNN-COMPs~\cite{brix2024fifth,kaulen20256th} (6h per benchmark).


  \subsection{\tool{} performance compared to baselines}
    \label{sec:result_baseline}
    \begin{table}[t]
      \centering
      \small
      \setlength{\tabcolsep}{5pt}
      \caption{Results across benchmarks (instances solved by at least one method).}
      \vspace{6pt}
      \begin{tabular}{c ccc ccc ccc}
        \toprule
        & \multicolumn{3}{c}{\textbf{FNN}} & \multicolumn{3}{c}{\textbf{CNN}} & \multicolumn{3}{c}{\textbf{Sigmoid}} \\
        \midrule
        \textbf{Method}
        & \makecell{\textbf{Time}\\(s)} & \makecell{\textbf{Branch}\\($\times 10^6$)} & \makecell{\textbf{Solved}\\(\#)}
        & \makecell{\textbf{Time}\\(s)} & \makecell{\textbf{Branch}\\($\times 10^6$)} & \makecell{\textbf{Solved}\\(\#)}
        & \makecell{\textbf{Time}\\(s)} & \makecell{\textbf{Branch}\\($\times 10^6$)} & \makecell{\textbf{Solved}\\(\#)} \\
        \cmidrule(lr){1-1} \cmidrule(lr){2-4} \cmidrule(lr){5-7} \cmidrule(lr){8-10}
        \textsc{Polarity} & 6064.32 & 170.81 & 1 & 8662.55 & 583.01 & 0 & 6027.70 & 62.98 & 10 \\
        \textsc{Upb} & 4359.36 & 27.17 & 23 & 2387.48 & 1.73 & 71 & 3963.69 & 8.59 & 39 \\
        \textsc{Fsb} & 2679.00 & 6.65 & 43 & \textbf{2295.89} & 1.81 & 71 & 4336.73 & 8.99 & 34 \\
        \textsc{Rsb} & \textbf{2051.94} & \textbf{1.86} & \textbf{50} & 2419.58 & \textbf{1.69} & \textbf{72} & \textbf{3450.80} & \textbf{5.14} & \textbf{43} \\
        \bottomrule
      \end{tabular}
      \label{tab:eval:size}
    \end{table}

    \autoref{tab:eval:size} presents our results.
    \tool{} solves the most instances across all three benchmarks (165 total vs.\ 148 for \fsb{}, 110 for \upb{}, and 11 for \polarity{}), an overall improvement of 11\% over \fsb{}, while exploring 50\% fewer branches compared to \fsb{}.
    These 11\% gain in solved instances and reduction in branch exploration show that \tool{} allows far fewer splits by selecting ``right'' neurons that tighten bounds faster and prune infeasible subproblems earlier.

    On FNN, \tool{} verifies 50 instances compared to 43 for \fsb{} (16\% more) while reducing total branch exploration by 72\% (1.86M vs.\ 6.65M).
    On Sigmoid, the gains are even larger: 43 vs.\ 34 solved instances (26\% more) with 43\% fewer branches (5.14M vs.\ 8.99M).
    On CNN, the improvement is smaller (72 vs.\ 71 solved, 7\% fewer branches), as CNNs have more neurons (\autoref{tab:benchmarks}) and both methods either timeout or require many branches to verify.
    Still, \tool{} reduces branches and solves one more instance.
    These results show that \tool{} adjusts \fsb{}'s decisions to explore better paths.

    \polarity{} solves only 11 instances in total despite exploring far more branches than \fsb{} and \tool{}. This shows that exploring more branches does not lead to more solved instances and that the branching heuristic has a large impact on verification performance.

    We note that the chosen baselines are already strong heuristics, \eg, \fsb{} being used in leading verifiers such as \crown{}~\cite{wang2021beta}, \gcpcrown{}~\cite{zhang2022general}, and \neuralsat{}~\cite{duong2024harnessing}.
    We consider improving over such well-established baselines is significantly more challenging, and these gains are meaningful progress in pushing the boundaries of DNN verification.
    More importantly, as a learning-based method, \tool{} continues to improve with more training data, better DRL algorithms, and new network architectures, unlike fixed hand-crafted heuristics.

  \subsection{\tool{}'s generalizability}\label{sec:generalizability}
    \begin{figure}[t]
      \centering
      \begin{minipage}[b]{0.5\linewidth}
        \centering
        \begin{subfigure}[t]{0.5\linewidth}
          \centering
          \includegraphics[width=\linewidth]{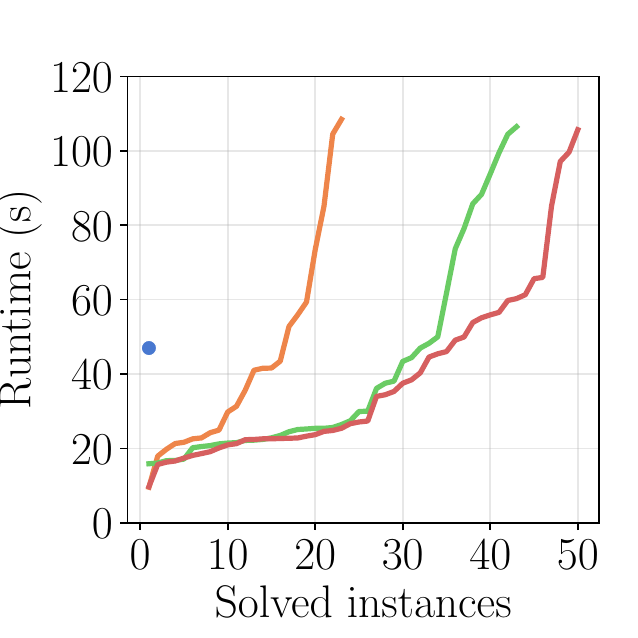}
          \caption{Total runtime.}
          \label{fig:eval:cactus:total_time:seen}
        \end{subfigure}%
        \hfill
        \begin{subfigure}[t]{0.5\linewidth}
          \centering
          \includegraphics[width=\linewidth]{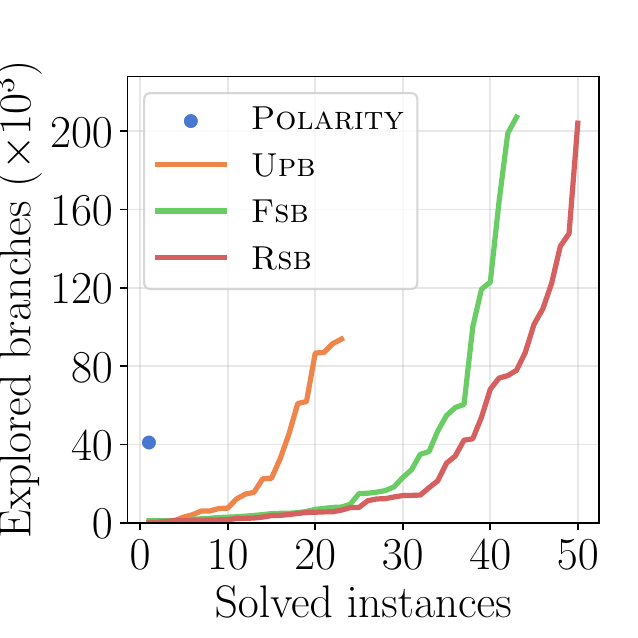}
          \caption{Explored branches.}
          \label{fig:eval:cactus:value:seen}
        \end{subfigure}%
        \caption{Performances on seen networks.}
        \label{fig:eval:cactus:seen}
      \end{minipage}%
      \hfill
      \begin{minipage}[b]{0.5\linewidth}
        \centering
        \begin{subfigure}[t]{0.5\linewidth}
          \centering
          \includegraphics[width=\linewidth]{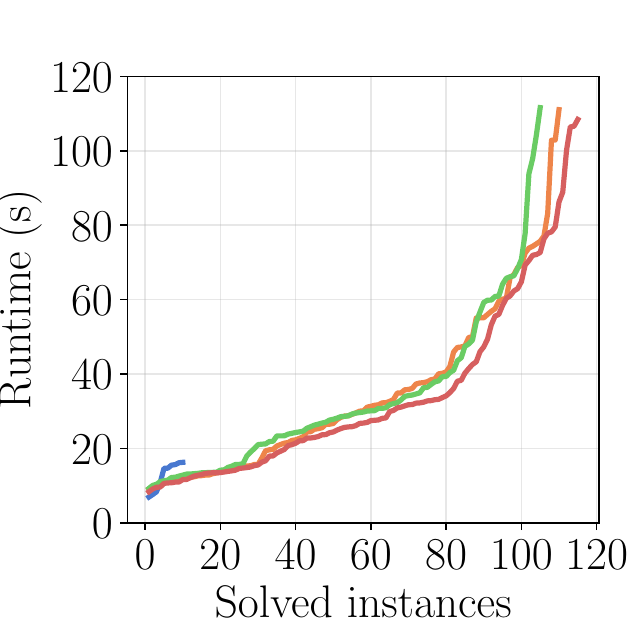}
          \caption{Total runtime.}
          \label{fig:eval:cactus:total_time:unseen}
        \end{subfigure}%
        \hfill
        \begin{subfigure}[t]{0.5\linewidth}
          \centering
          \includegraphics[width=\linewidth]{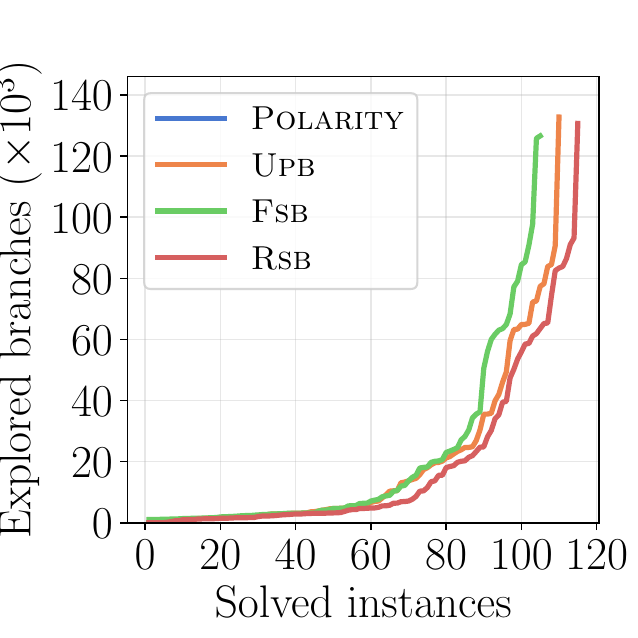}
          \caption{Explored branches.}
          \label{fig:eval:cactus:value:unseen}
        \end{subfigure}%
        \caption{Performances on unseen networks.}
        \label{fig:eval:cactus:unseen}
      \end{minipage}
    \end{figure}

    \autoref{fig:eval:cactus:seen} illustrates the performance of \tool{} on seen networks (FNN).
    Evaluation instances share the same architectures as training instances but are paired with different properties.
    Since a \emph{verification problem} is a (network, property) pair, different properties yield different hidden-layer bounds and different BaB trees---they are distinct problems---and thus train and test sets do not overlap.

    \tool{}'s line (red) extends furthest on the x-axis (50 instances) and runs below \fsb{} throughout both runtimes (\autoref{fig:eval:cactus:total_time:seen}) and explored branches (\autoref{fig:eval:cactus:value:seen}), which shows that \tool{} solves more instances while consistently requiring fewer branches or runtimes.
    As problem difficulty increases, the gap between \tool{} and \fsb{} widens.
    \tool{}'s curve stays lower while \fsb{}'s rises faster, indicating that the learned policy becomes more effective on hard instances where branching order has a larger impact.

    \autoref{fig:eval:cactus:unseen} shows the performances on unseen architectures (CNN and Sigmoid).
    \tool{}'s line extends furthest (115 instances), with \upb{} second (110) and \fsb{} third (105).
    An interesting observation is that \fsb{} performs worse than \upb{} on unseen architectures.
    However, \tool{} adjusts \fsb{}'s scores through learned weights, and its curve remains consistently below \upb{}'s across the entire range, recovering from \fsb{}'s weakness and surpassing all baselines.
    This shows that \tool{} generalizes beyond its training distribution, learning branching strategies that transfer to unseen network architectures.

  \subsection{\tool{}'s efficiency}
    \begin{figure}[t]
      \centering
      \begin{subfigure}[t]{0.3\linewidth}
        \centering
        \includegraphics[width=\linewidth]{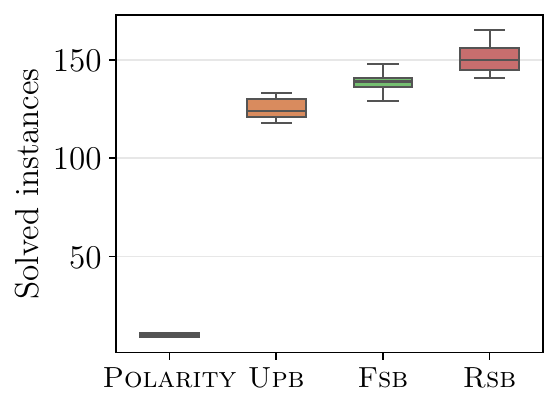}
        \caption{Solved instances.}
        \label{fig:eval:box:solved}
      \end{subfigure}%
      \hfill
      \begin{subfigure}[t]{0.3\linewidth}
        \centering
        \includegraphics[width=\linewidth]{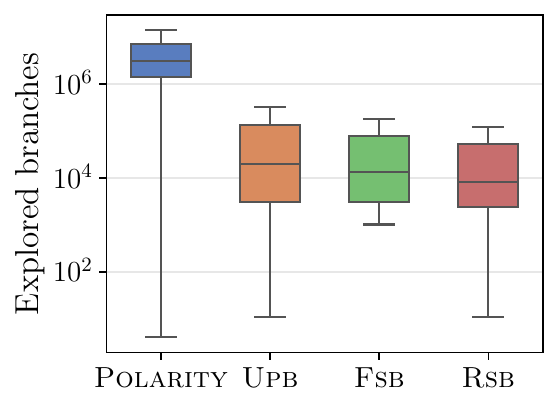}
        \caption{Explored branches.}
        \label{fig:eval:box:branches}
      \end{subfigure}%
      \hfill
      \begin{subfigure}[t]{0.3\linewidth}
        \centering
        \includegraphics[width=\linewidth]{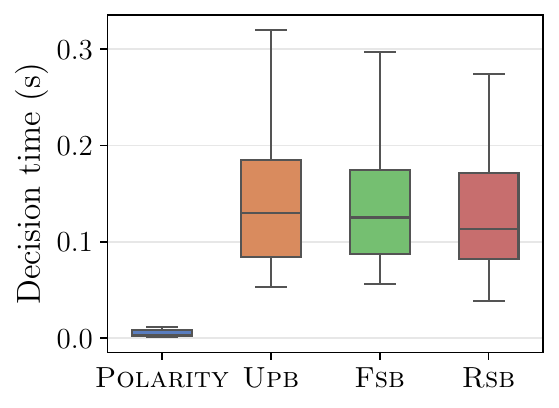}
        \caption{Decision time.}
        \label{fig:eval:box:decision_time}
      \end{subfigure}%
      \caption{Statistical distributions of computational metrics.}
      \label{fig:eval:box}
    \end{figure}

    To assess statistical significance (\autoref{fig:eval:box}), we evaluate each method on five runs that use different seeds for benchmark generation.
    \tool{} solves the most instances across all runs (\autoref{fig:eval:box:solved}), with a median of 150 and a tight interquartile range, compared to 140 for \fsb{} and 125 for \upb{}.
    The narrow spread of \tool{}'s box confirms that its advantage is consistent and not driven by a single run.

    We further show the branch distribution in log scale in \autoref{fig:eval:box:branches}.
    \polarity{} explores orders of magnitude more branches ($\sim$$10^6$) than the rest, yet solves the fewest instances, which clearly indicates that exploring more branches does not compensate for poor branching decisions.
    \tool{} and \fsb{} both have medians near $10^4$, but \tool{}'s box is lower and its spread is tighter, showing that it reaches solutions with fewer branches across instances of varying difficulty.

    \autoref{fig:eval:box:decision_time} shows that \tool{}'s per-step decision time ($\sim$0.11s) is lower than \fsb{} and \upb{} ($\sim$0.13s) despite performing additional DRL inference at every step.
    Both \tool{} and \fsb{} run \fsb{} scoring internally, but \tool{}'s better branching keeps the subproblem queue smaller (\eg, 32 vs.\ 64 active subproblems per batch).
    As a result, \tool{} runs \fsb{} on fewer subproblems per step, and the time saved from the smaller batch offsets the added DRL inference cost.

\section{Related work}
  Abstraction in \babverifier{} plays a key role in scalability, \eg, intervals~\cite{wang2018formal}, zonotopes~\cite{singh2018fast}, polytopes~\cite{singh2019abstract,zhang2018efficient,xu2020automatic}, and star sets~\cite{tran2019star} are widely used by leading verifiers~\cite{brix2023first,brix2024fifth,kaulen20256th}.
  Bound tightening is further strengthened by multi-neuron relaxations~\cite{ferrari2022complete}, cutting planes~\cite{zhang2022general, zhou2024scalable}, and neuron stability~\cite{duong2024harnessing},
  while training for verification reduces the number of unstable neurons in the first place~\cite{xu2024training}.
  \tool{} focuses on the branching decision and complements these approaches. 

  \fsb{}~\cite{de2021improved} is the SOTA heuristic that estimates bound improvements via Lagrangian Decomposition, outperforming \babsr{}, \polarity{}, and GNN heuristic~\cite{lu2020neural}.
  \fsb{} estimates per-neuron branching scores using tighter bounds from the dual problem, making it a strong greedy heuristic.
  \upb{}~\cite{de2022ibp} is an approximation of \fsb{}~\cite{de2021improved} by reusing dual variables computed during the bounding steps~\cite{wang2021beta}.
  \babsr{}~\cite{bunel2020branch} uses an abstraction to approximate bound improvements per neuron while \polarity{}~\cite{wu2020parallelization} selects neurons based solely on their bounds.
  GNN heuristic~\cite{lu2020neural} takes a learning-based approach with bidirectional updates to predict decisions, and is closely related to our work.
  However, these heuristics make greedy decisions, scoring each neuron based on immediate improvements without considering long-term impact on search tree.
  A neuron that maximizes immediate bounds may still produce subproblems that require more branches to resolve, resulting in deep search trees.

  In SAT solving, prior DRL-based approaches~\cite{kurin2020can, wang2021neuroback, zhai2025learning} require a separate GNN forward pass at each decision step, incurring significant overhead.
  The one-shot guidance approach~\cite{tonshoff2025learning} reduces this cost by applying a single GNN forward pass whose predictions guide all subsequent steps, achieving a 2$\times$ speedup and generalizing to larger problems.

  Our work addresses the greedy selections~\cite{bunel2020branch, de2021improved} by generating weights to refine baseline heuristic scores with learned policies.
  Unlike these methods, we use raw neuron features for architecture-agnostic verification (\autoref{sec:training:embed}). 
  We also adapt the training procedure to dual-branch dynamics, which complements the one-shot guidance approach~\cite{tonshoff2025learning} (\autoref{sec:training:procedure}).

\section{Conclusion}
  \label{sec:conclusion}
  We demonstrated that the proposed DRL branching heuristics can effectively learn policies that outperform mature heuristics in DNN verification.
  Experimental results validate that DRL can capture long-term search strategies and outperform greedy heuristics.
  More broadly, learning-based paradigm opens new directions for improving verification efficiency.
  As DRL algorithms advance~\cite{schulman2017proximal,haarnoja2018soft,levine2020offline,le2024multitask} and training data grows, performance can continue to improve.
  Future work includes extending \tool{} to input-space branching~\cite{wang2018formal} for low-dimensional inputs and online adaptation during verification~\cite{bello2016neural} to enhance the performance of the learned policy on hard instances. 

\bibliographystyle{unsrtnat}
\bibliography{paper}

\appendix

\section{The BaB algorithm}\label{apdx:bab_algorithm}
  \begin{algorithm}[htbp]
    \small
    \caption{The \babverifier{} algorithm}
    \label{alg:babnnv}
    \Input{DNN $N$, property $\phi_{in} \Rightarrow \phi_{out}$}
    \Output{\unsat{} if property is valid, otherwise \sat{}}
    \BlankLine
    $\problems \leftarrow \{ \emptyset \}$ \\
    \While{$\problems \neq \emptyset$}{
      $\sigma \gets \Select(\problems)$ \\
      \If{\Deduce($N, \phi_{in}, \phi_{out}, \sigma$)}{ \label{line:deduce}
        $(\cex, v_i) \leftarrow \Decide(N, \phi_{in}, \phi_{out}, \sigma)$ \label{line:decide} \\
        \If{$\cex$}{
          \Return{\sat{}}
        }
        $\problems \leftarrow \problems \cup \{ \sigma \land v_i ~;~ \sigma \land \overline{v_i} \}$ \\
      }
    }
    \Return{\unsat{}}
  \end{algorithm}

  \autoref{alg:babnnv} shows \babverifier{}, a reference architecture~\cite{nakagawa2014consolidating} for modern DNN verifiers.
  \babverifier{} takes as input a DNN $N$ and a property $\phi_{in}\Rightarrow \phi_{out}$.
  \babverifier{} iterates between two components: \Decide{} (branching, \autoref{line:decide}) assigns active/inactive status for a neuron, and \Deduce{} (bounding, \autoref{line:deduce}) checks the feasibility of the current subproblem.
  Both components are critical to verification efficiency, but this work focuses on improving the \Decide{}.

\section{DRL details}\label{apdx:drl}
  We use a variant of soft actor-critic (SAC)~\cite{haarnoja2018soft} including two $Q$-networks and target $Q$-networks supplying bootstrapped value targets, the actor generates a distribution over candidate neurons, which yields policy updates close to deterministic policy-gradient methods but retains the off-policy sample reuse of SAC.

  \paragraph{Value objective.}
  The networks are trained to maximize the expected cumulative discounted return, summarized the role of each module via the state-value identity
  \begin{equation}
    V(s_t) = \mathbb{E}_{a_t \sim \pi}[Q_\phi(s_t, a_t)]
  \end{equation}

  \paragraph{Critic update.}
  The critic is trained using temporal difference learning to minimize the prediction error of value estimates:
  \begin{equation}
    \mathcal{L}_{Q}(\phi) = \mathbb{E}_{(s_t, a_t) \sim \mathcal{D}} \left[
      \frac{1}{2} \Big(
        Q_\phi(s_t, a_t) - \big(
          R(s_t, a_t) + \gamma \, \mathbb{E}_{s_{t+1} \sim P} [V_{\bar{\phi}}(s_{t+1})]
        \big)
      \Big)^2
    \right]
  \end{equation}
  where $\mathcal{D}$ is the replay buffer of collected transitions $\langle s_t, a_t, r_t, s_{t+1} \rangle$.
  The replay buffer enables off-policy learning by storing and reusing past experiences, which improves sample efficiency and stabilizes training.
  The discount factor $\gamma \in [0,1]$ controls the trade-off between immediate and future rewards.
  In particular, $\gamma = 0$ focuses only on immediate rewards, while $\gamma$ close to 1 encourages the agent to consider long-term consequences.

  \paragraph{Actor update.}
  The actor is updated using policy gradients:
  \begin{equation}
    \mathcal{L}_{\pi}(\psi) = \mathbb{E}_{s_t \sim \mathcal{D}} \left[
      \mathbb{E}_{a_t \sim \pi_\psi} \left[
        - Q_\phi(s_t, a_t)
      \right]
    \right]
  \end{equation}
  The actor minimizes this loss to guide action selection toward higher expected returns, as estimated by the critic.
  The critic learns to estimate action values from collected experience, while the actor uses those estimates to take better actions.

\section{Training details}\label{apdx:training}

\subsection{Training setup}

\paragraph{Replay, rewards, and buffer capacity.}
Specifications are listed in \autoref{tab:train:hyperparameters}.
Bounded replay limits memory usage while retaining diverse branching histories from previous problems.
When the buffer reaches capacity, the oldest transitions are removed in a first-in, first-out manner, balancing recent and diverse off-policy data.
Rewards are normalized by a fixed maximum number of subproblems that an episode may visit, assuring that return scales are comparable across instances with varying search-tree sizes.

\paragraph{Warm-up.}
During the initial $n_{\mathrm{start}}$ environment steps (branching steps summed over episodes), actions are selected using lightweight hand-crafted heuristics (\eg, \fsb{}, \random{}, etc.) instead of the learned policy.
This approach guarantees that early replay can be contained.
Optimization may also begin once the buffer is at least half full, even if $n_{\mathrm{start}}$ hasn't been reached.
After warm-up, letting $g$ index environment steps, we still take a heuristic action occasionally (\eg, every 5 steps).
Otherwise, we follow the actor.

\paragraph{Optimization choices (\autoref{tab:train:hyperparameters}).}
The actor and critic use learning rates on the order of $10^{-4}$, a scale commonly employed in off-policy actor-critic methods for long episodes.
The discount factor $\gamma$ is set to $0.99$, making sure that future branching depth is not excessively discounted relative to early decisions.
Minibatches contain a few hundred transitions.
Gradient accumulation across several draws per update approximates a larger batch and stabilizes the critic targets under GPU memory restrictions.
Multiple critic (Bellman) updates are applied before each actor update, allowing $Q$-values to more closely track the evolving policy.
Twin $Q$-networks and pessimistic targets are used to mitigate overestimation in value backups.
Target networks are updated using Polyak averaging with a coefficient $\tau = 10^{-4}$, resulting in slowly moving bootstrap targets.


\paragraph{Network architecture.}
The actor and critic share the same observations (\autoref{sec:training:embed}) but use separate PointerNet-based networks~\cite{bello2017neural}.
Each network applies linear projections to raw and embedded neuron features, an LSTM encoder and decoder, $n_{\mathrm{glimpse}}$ attention glimpses, and logits clipped to $[\pm10]$ before applying softmax over admissible neurons.
The critic outputs one scalar per branch-child pair using the dual-branch observation (\autoref{sec:training:embed}).

\begin{table}[t]
  \centering
\caption{Hyperparameters for RL training and for GCN embedding pretraining used in our experiments. The episode count equals the main paper's ($10^5$ unless stated otherwise).}
\label{tab:train:hyperparameters}\label{tab:train:hyperparameter}
\begin{tabular}{l l}
  \toprule
  \textbf{Hyperparameter} & \textbf{Value} \\
  \midrule
  \multicolumn{2}{@{}l}{\emph{RL training (Actor and Critic)}} \\
  Optimizer & AdamW \\
  Actor learning rate & $4\times 10^{-4}$ \\
  Critic learning rate & $4\times 10^{-4}$ \\
  Discount factor $\gamma$ & $0.99$ \\
  Replay buffer capacity & $10^5$ transitions \\
  Minibatch size & $512$ \\
  Gradient accumulation & $8$ minibatches per gradient step \\
  Critic updates per policy update & $2$ \\
  Warm-up step $n_{\mathrm{start}}$ & $10^3$ steps \\
  Soft target update $\tau$ & $10^{-4}$ \\
  Target network update & 1 step \\
  Training episodes & $10^5$ \\
  Checkpoint period & $10^4$ steps \\
  \midrule
  \multicolumn{2}{@{}l}{\emph{GCN embedding (before RL training)}} \\
  Architecture & 2-layer GCN, hidden width $128$, ReLU \\
  Node feature & $4$ (lower bound, upper bound, bias, unstable mask) \\
  Supervision & MSE on \fsb{} scores of unstable neurons \\
  Optimizer & Adam \\
  Learning rate & $3\times 10^{-4}$ \\
  Epochs & $5$ \\
  Mini-batch size & $512$ \\
  \bottomrule
\end{tabular}
\end{table}


\subsection{DRL training progress}

\begin{figure}[t]
  \centering
  \includegraphics[width=0.6\textwidth]{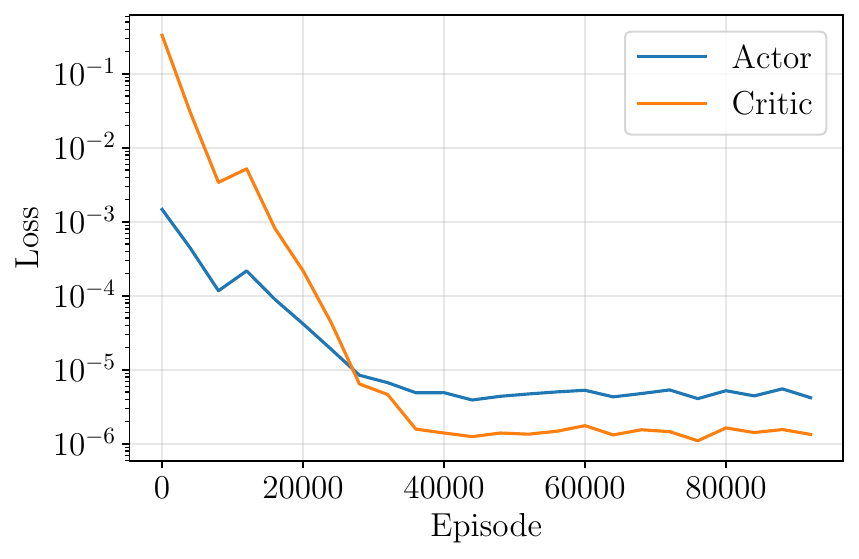}
  \caption{Training loss over time.}
  \label{fig:train:loss}
\end{figure}

\autoref{fig:train:loss} shows actor/critic losses alongside training; both plateau without divergence after roughly $10^6$ environment steps in our recorded run.
Optional checks, such as periodically running a fixed verification suite to measure branch counts during training, provide a closer proxy to deployment and match our checkpoint schedule.
Training on a single GPU for approximately $10^5$ verification episodes (batch size $512$ with gradient accumulation $8$) requires several days.
After training, model weights are frozen and shared across all test instances without per-network fine-tuning.
Wall-clock time comparisons under competition time budgets are reported in the main text alongside branch counts.

\end{document}